\documentclass[letterpaper, 10 pt, conference]{ieeeconf}  % Comment this line out if you need a4paper

\IEEEoverridecommandlockouts                              % This command is only needed if 
\usepackage{graphicx} 
\usepackage{morefloats}
\usepackage{subfig}
\usepackage{bm}
\usepackage[boxruled,linesnumbered]{algorithm2e}
\usepackage{mathrsfs}
\usepackage{amssymb}
\usepackage{amsmath}
\usepackage{epsfig}
\usepackage{xcolor}

\newtheorem{theorem}{Theorem}{\bfseries}{\itshape} 
\newtheorem{assumption}{Assumption}{\bfseries}{\itshape} 

\renewcommand{\vec}{\mathbf}

\newcommand{\argmin}[1]{\underset{#1}{\operatorname{argmin}}\;}

\title{\LARGE \bf
Interleaving Optimization with Sampling-Based Motion Planning (IOS-MP): Combining Local Optimization with Global Exploration
}

\author{Alan Kuntz, Chris Bowen, and Ron Alterovitz% <-this % stops a space
\thanks{This research was supported in part by the National Science Foundation (NSF) under award IIS-1149965. Any opinions, findings, and conclusions or recommendations expressed in this material are those of the authors and do not necessarily reflect the views of NSF.}% <-this % stops a space
\thanks{Alan Kuntz, Chris Bowen, and Ron Alterovitz are with the Department of Computer Science,
        University of North Carolina at Chapel Hill, NC 27599-3175, USA
        {\tt\small \{adkuntz, cbbowen, ron\}@cs.unc.edu}}%
}

\begin{document}

\maketitle
\thispagestyle{empty}
\pagestyle{empty}

\begin{abstract}
	
Computing globally optimal motion plans for a robot is challenging in part because it requires analyzing a robot's configuration space simultaneously from both a macroscopic viewpoint (i.e., considering paths in multiple homotopic classes) and a microscopic viewpoint (i.e., locally optimizing path quality).
 We introduce Interleaved Optimization with Sampling-based Motion Planning (IOS-MP), a new method that effectively combines global exploration and local optimization to quickly compute high quality motion plans.
 Our approach combines two paradigms: (1) asymptotically-optimal sampling-based motion planning, which is effective at global exploration but relatively slow at locally refining paths, and (2) optimization-based motion planning, which locally optimizes paths quickly but lacks a global view of the configuration space.
 IOS-MP iteratively alternates between global exploration and local optimization, sharing information between the two, to improve motion planning efficiency.
We evaluate IOS-MP as it scales with respect to dimensionality and complexity, as well as demonstrate its effectiveness on a 7-DOF manipulator for tasks specified using goal configurations and workspace goal regions.

\end{abstract}

%%%%%%%%%%%%%%%%%%%%%
% INTRODUCTION
%%%%%%%%%%%%%%%%%%%%%
% !TEX root =  Kuntz_IOSMP.tex

\section{Introduction}

Robots are increasingly entering domains such as transportation, surgery, and home assistance where safe interaction with people is critical. 
This interaction motivates robots which are capable of planning high quality motions under short time horizons.
Unfortunately, the landscape of feasible motion plans in a robot's configuration space can be extremely complex, consisting of many local minima, with paths spanning multiple homotopic classes.
The complexity of this landscape means that an ideal motion planning algorithm must employ a macroscopic global view, considering paths in multiple homotopic classes, while also taking a microscopic local view, ensuring plans are as close to locally optimal as possible. 
Our work focusses on unifying these two perspectives by interleaving path refinement through local optimization, with global exploration through sampling-based motion planning.
In this way, we compute high quality, locally optimized plans while continuing to explore the global landscape for as long as time allows.

Sampling-based motion planning methods typically take the macroscopic view, drawing samples from the robot's entire configuration space to build a graph that progressively covers more and more of the space.
The global nature of these algorithms allows them to explore every relevant homotopic class eventually.
Many of these methods will also converge upon the globally optimal solution as the number of samples approaches infinity \cite{Karaman2011_IJRR}\cite{Gammell2015_ICRA}.
Unfortunately, the resulting paths may converge slowly, so in \emph{finite} time, they frequently return paths that are far from locally optimal, especially in higher dimensional problems \cite{Janson2015_IJRR}.

\begin{figure}[t]
  \centering
  \subfloat[First Iteration]{\label{fig:homotopy1}\includegraphics[width=0.33\linewidth]{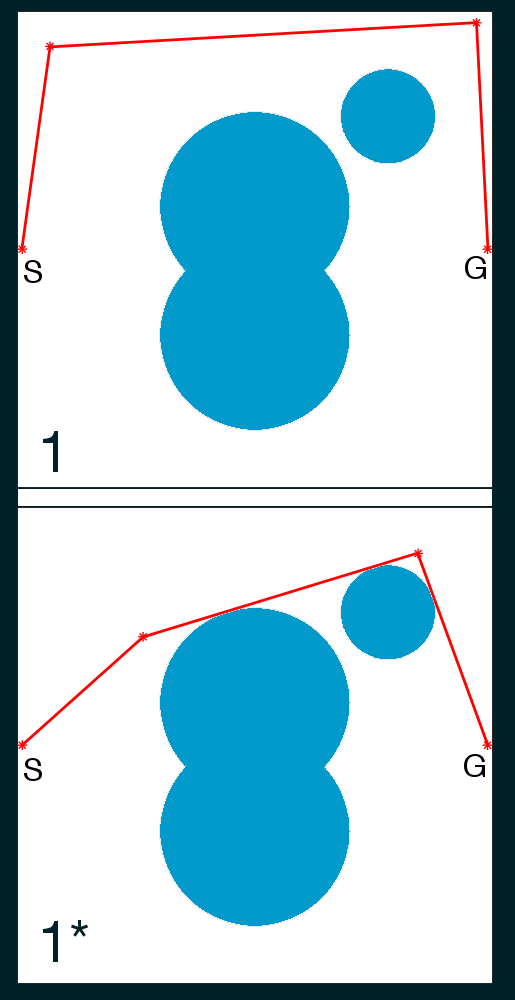}}
  \subfloat[Second Iteration]{\label{fig:homotopy2}\includegraphics[width=0.33\linewidth]{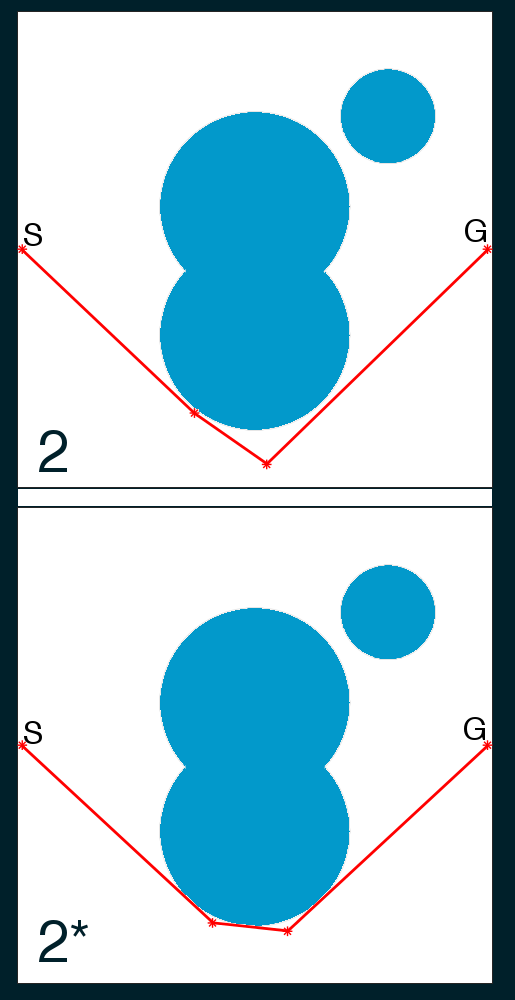}}
  \subfloat[Third Iteration]{\label{fig:homotopy3}\includegraphics[width=0.33\linewidth]{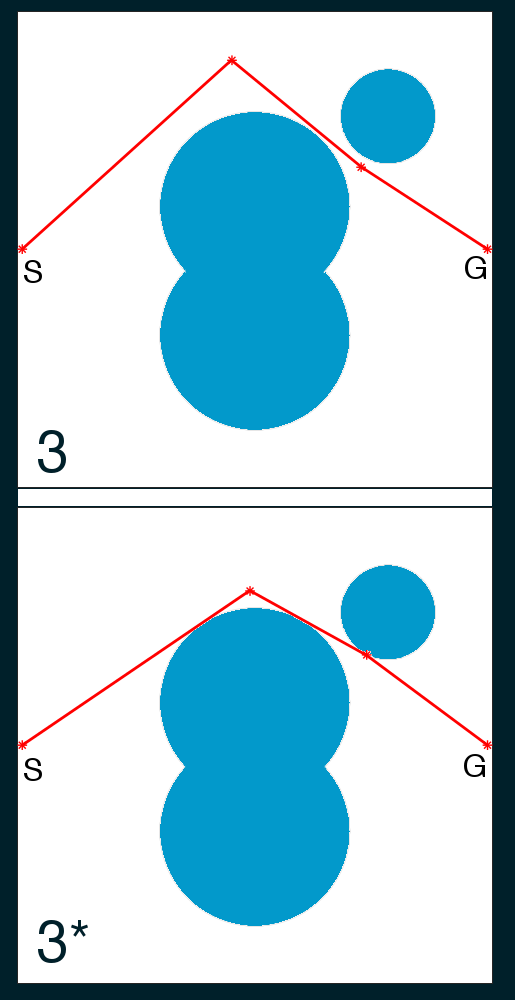}}
  \caption{Iterations of IOS-MP progress toward a globally optimal path. (a) The sampling-based motion planner finds the first path (top), which is then locally optimized (bottom). 
(b, c) In each subsequent iteration, the sampling-based motion planning step  finds a path that is shorter than the previously found best path (top), and this new path is then locally optimized (bottom) before continuing to the next iteration.}
    \label{fig:fig1}
\vspace{-12pt}
\end{figure}

Optimization-based motion planning methods take a more microscopic approach.
These methods quickly compute high quality paths by numerically optimizing an initial path, converging toward a local minimum \cite{Ratliff2009_ICRA}\cite{Schulman2013_RSS}\cite{Park2012_ICAPS}\cite{Kalakrishnan2011_ICRA}.
However, the resulting path quality of such methods is directly tied to the path initialization. 
Due to the inherently non-convex nature of motion planning with obstacles, for some initializations the resulting paths may be far from globally optimal, or the optimization may not find a solution at all.
This limitation can be partially circumvented through techniques such as restarting the optimization with multiple different initial paths (e.g., \cite{Schulman2013_RSS}), but such approaches are heuristically driven and typically provide no global guarantees.
Additionally, by treating each restart as independent, they potentially ignore useful pre-computed information (such as collision-free configurations) that could be valuable if shared across restarts.

Our new method, Interleaved Optimization with Sampling-based Motion Planning (IOS-MP), alternates between local optimization and global exploration, sharing information between the two, effectively combining the best of both of these paradigms (see Fig. \ref{fig:fig1}).
IOS-MP starts with global exploration by building a graph using an asymptotically optimal motion planner, such as $k$-nearest Probabilistic Roadmap (PRM*) \cite{Karaman2011_IJRR} or Batch Informed Trees (BIT*) \cite{Gammell2015_ICRA}, until it finds a collision free path (Fig. \ref{fig:fig1}(a) top).
IOS-MP then uses constrained local optimization based on an augmented Lagrangian method \cite{Wright_1999} to refine the path found by the sampling-based method (Fig. \ref{fig:fig1}(a) bottom).
The algorithm then iterates between (1) resuming the sampling-based motion planner until a new better path is found and (2) running constrained local optimization on this new path (see Fig. \ref{fig:fig1}(b),(c)).
The sampling-based motion planning phase of each iteration explores globally, discovering other homotopic classes in configuration space, as well as escaping local minima in the path cost landscape.
The constrained local optimization phase in turn allows the method to provide a high quality local solution at the end of each iteration. 
For efficiency, information is shared by both of the phases at each iteration, with the sampling-based method seeding the local optimization with improving initial solutions, and the local optimization passing potentially valuable new vertices and edges to the sampling-based method.

IOS-MP is an \emph{anytime} motion planning algorithm, in that the algorithm can be stopped at any time (after iteration 1) and return a locally optimized path, and running the algorithm for more time will return locally optimized paths that asymptotically approach a globally optimal path. 
The framework we introduce is generalizable; while our implementation uses augmented Lagrangian optimization and PRM* or BIT*, other local optimization algorithms and other sampling-based planning algorithms can be substituted in as long as the formulations and constraints described in this paper can be applied.
The contribution of this paper is in describing and evaluating a framework for constructing a motion planning algorithm which leverages both local path optimization and global exploration.
This allows it to provide higher quality paths earlier than asymptotically optimal sampling-based motion planning algorithms alone can provide, while providing the guarantees---namely completeness and asymptotic optimality---which are not provided by most optimization-based motion planning algorithms.

In this paper we optimize trajectories with respect to path length. 
We demonstrate IOS-MP's efficacy in simulated environments of varying dimensionality and complexity as well as for a 7 degree of freedom (DOF) manipulator performing tasks defined by both goal configurations and workspace goal regions.

%%%%%%%%%%%%%%%%%%%%%%%%%%%%
% RELATED WORK
%%%%%%%%%%%%%%%%%%%%%%%%%%%%
% !TEX root =  Kuntz_IOSMP.tex

\section{Related Work}

In sampling-based motion planning, a graph data structure is constructed incrementally via random sampling providing a collision-free tree or roadmap in the robot's configuration space.
Early versions of these algorithms provide probabilistic completeness, i.e., the probability of finding a path, if one exists, approaches one as the number of samples approaches infinity.
Classic examples include the Rapidly-exploring Random Tree (RRT) \cite{LaValle2001_IJRR} and the Probabilistic Roadmap (PRM) \cite{Kavraki1996_TRA} methods.
Adaptations of these algorithms can provide asymptotic optimality guarantees, wherein the path cost (e.g., path length) will approach the global optimum as the number of algorithm iterations increases.
Examples include RRT* and PRM* \cite{Karaman2011_IJRR} where the underlying motion planning graph is either rewired or has asymptotically changing connection strategies.
Other asymptotically optimal algorithms include Batch Informed Trees (BIT*) \cite{Gammell2015_ICRA} in which samples are processed in batches and Fast Marching Trees (FMT*) \cite{Janson2015_IJRR} which grows a tree in cost-to-arrive space. 
Other work such as the cross-entropy method \cite{Kobilarov2012_IJRR} has investigated the distributions from which samples and trajectories are taken.

Optimization-based motion planning algorithms perform numerical optimization in a high dimensional trajectory space.
Each trajectory is typically represented by a vector of parameters representing a list of robot configurations or controls.
A cost can be computed for each trajectory, and the motion planner's goal is to compute a trajectory that minimizes cost.
In the presence of obstacles and other constraints, the problem can be formulated and solved as a numerical optimization problem.
CHOMP \cite{Ratliff2009_ICRA} takes an initial trajectory and performs gradient descent. Traj-Opt \cite{Schulman2013_RSS} uses sequential quadratic programming with inequality constraints to locally optimize trajectories. ITOMP combines optimization with re-planning to account for dynamic obstacles \cite{Park2012_ICAPS}.
These methods typically produce high quality paths but are frequently unable to escape local minima, and as such are subject to initialization concerns.
To avoid local minima, some methods inject randomness into the system \cite{Kalakrishnan2011_ICRA,Carriker1990_ICRA}.

Several methods aim to bridge the gap between sampling-based methods and optimization. 
Some methods use paths generated by global planners and refine them using shortcutting or smoothing methods in post processing or adaptively  \cite{Kavraki1996_TRA}\cite{Pan2012_IJRR}\cite{Hauser2010_ICRA}\cite{Luna2013_ICRA}\cite{Quinlan1993_ICRA}.
Recent work has begun considering ways to integrate ideas from optimization-based motion planning with sampling-based motion planning.
GradienT-RRT moves vertices to lower cost regions using gradient descent during the construction of an RRT \cite{Berenson2011_ICRA}.
More recently, RABIT* uses CHOMP to locally optimize in-collision edges during BIT* planning to bring them out of collision and effectively find narrow passages \cite{Choudhury2016_ICRA}.
In contrast, our method is using the local optimizer not to find narrow passages, but to improve the overall quality of the paths found by the sampling-based planner, while relying on the sampling-based planner's completeness property to discover narrow passages.

IOS-MP aims to more thoroughly integrate optimization with sampling-based motion planning by interleaving them to achieve fast motion planning with asymptotic optimality.

%%%%%%%%%%%%%%%%%%%%%%%
% PROBLEM FORMULATION
%%%%%%%%%%%%%%%%%%%%%%%
% !TEX root =  Kuntz_IOSMP.tex

\section{Problem Definition}

Let $C$ be the configuration space of the robot. 
Let $\vec{q} \in C$ represent a single robot configuration and $\vec{p} = \{\vec{q}_0,\vec{q}_1, \dots, \vec{q}_n\}$ be a continuous path in configuration space represented in a piecewise manner by a sequence of configurations. 
In this paper, we minimize the path with respect to the length, where we define $\texttt{length}(\vec{p})$ to be the sum of Euclidean distances in configuration space along the path.
We use the sum of Euclidean distances because it is a commonly used metric for path length and because it satisfies the triangle inequality property required by asymptotically optimal sampling-based motion planners (unlike other metrics such as sum of squared distances). 
While we use path length for our optimization objective, certain other cost functions that are continuous, differentiable, and do not violate the triangle inequality could be used instead.

In the robot's workspace are obstacles that must be avoided.
Each obstacle may be composed of a set of obstacle primitives, and let the set of primitives composing all obstacles be $O$. 
Let $C_\textrm{free} \subseteq C$ be the collision free subspace and $C_\textrm{obs} \subseteq C$ be the in-collision subspace of $C$ based on obstacle primitives $O$.
In the context of this work, we are considering static environments, such that $C_\textrm{obs}$ is unchanging as a function of time.
A path is collision-free if each edge $(\vec{q}_i, \vec{q}_{i+1}), i=0,\ldots,|\vec{p}|-1$, does not intersect an obstacle primitive $o \in O$. 
Formally, we define $\texttt{clearance}(\vec{q}_i, \vec{q}_{i+1}, o)$ as the signed squared distance between the path \emph{edge} $(\vec{q}_i, \vec{q}_{i+1})$ and obstacle primitive $o$ (where negative values correspond to obstacle penetration distance); a collision-free path is a path for which each edge has non-negative clearance.
This requires that the obstacle and robot representations allow for the computation, either analytically or numerically, of the signed distance between the robot and the obstacle primitive. 
This property holds for robots and obstacles represented using collections of bounding spheres or bounding capsules (volumes defined by a sphere swept along a straight line segment) as well as point clouds and other representations, which are common representations used in the literature and useful in practice.

The optimal motion planning problem then becomes finding a collision-free path from $\vec{q}_\textrm{start}$ to $\vec{q}_\textrm{goal} \in C_\textrm{goal}$ that minimizes cost, where $C_\textrm{goal}$ is the (possibly singleton) set of goal configurations. This optimal motion planning problem can be formulated as a nonlinear, constrained optimization problem:
\begin{equation}
\begin{array}{ll}
\multicolumn{2}{l}{\displaystyle{\vec{p}^{*} = \argmin{\vec{p}} \texttt{length}(\vec{p})}} \\
\mbox{Subject to:}                                         & \ \\
\hspace{0.2in}   \texttt{clearance}(\vec{q}_i, \vec{q}_{i+1}, o) \geq 0, \hspace{2pt}     & 0 \le i < |\vec{p}|, \hspace{2pt} \forall \hspace{2pt} o \in O \\
\hspace{0.2in} \vec{g}(\vec{p}) \geq 0, \hspace{16pt}    & \forall \hspace{2pt} \vec{g} \in J \\
\hspace{0.2in}   \vec{q}_0 = \vec{q}_\textrm{start}     & \ \\
\hspace{0.2in}   \vec{q}_{|\vec{p}|} = \vec{q}_\textrm{goal}     &  \vec{q}_\textrm{goal} \in C_\textrm{goal}\ \\
\end{array}
\label{eqn:PlanningFormulation}
\end{equation}
where $J$ is a set of generalized inequality constraints that are specific to the problem and robot (e.g., to represent joint angle limits for a manipulator), and where $\vec{p}^*$ is the optimal motion plan.
Let $K$ be the set of inequality constraints implied by $O$ for obstacle avoidance.
The set of all inequality constraints then becomes $I = J \bigcup K$.
IOS-MP is an iterative algorithm for efficiently solving this problem such that the solution asymptotically approaches $\vec{p}^{*}$.

%%%%%%%%%%%%%%%%%%%%%%%%%%%%
% METHOD
%%%%%%%%%%%%%%%%%%%%%%%%%%%%
% !TEX root =  Kuntz_IOSMP.tex

\section{Method}

IOS-MP integrates ideas from both sampling-based and optimization-based motion planning. The method interleaves path optimization steps with graph expansion steps to achieve fast convergence.

\subsection{Overview}

The top level algorithm for IOS-MP, Alg. \ref{alg:top}, runs in an anytime manner, iterating as time allows and storing the best path found up to any time.

\IncMargin{0.5em}
\begin{algorithm}
\KwIn{start configuration $\vec{q}_\mathrm{start}$, obstacle set $O$, time limit $t$}
\KwOut{motion plan $\vec{p}^{*}$}
$G \leftarrow (\{\vec{q}_\mathrm{start}\}, \emptyset)$ \\
Best cost $c \leftarrow \infty$ \\
$\vec{p}^{*} \leftarrow \emptyset$\\
\While{\emph{time elapsed} $\le t$} {
$\vec{p} \leftarrow \texttt{graphExpansionStep}(O, G, c, t)$\\
$c \leftarrow \texttt{cost}(\vec{p})$\\
$\vec{\hat{p}} \leftarrow \texttt{optimizationStep}(\vec{p})$\\
$c \leftarrow \texttt{cost}(\vec{\hat p})$\\
$G \leftarrow G \bigcup \vec{\hat p}$\\
$\vec{p}^{*} \leftarrow \hat{\vec{p}}$\\
}
\caption{IOS-MP}
\label{alg:top}
\end{algorithm}
\DecMargin{0.5em}

In the first step of each iteration, IOS-MP executes a sampling-based motion planner to expand the graph $G$ until a new path is found that has cost lower than $c$. The sampling-based motion planner returns the new path $\vec{p}$ and updates $c$. 
In the second step of each iteration, IOS-MP executes a local optimization method to locally optimize $\vec{p}$. The method saves the optimized path $\hat{\vec{p}}$ as the best new motion plan found. It also adds the configurations and segments defining $\hat{\vec{p}}$ as new vertices in $G$. 
The algorithm then iterates, returning to global exploration with the sampling-based motion planner, but seeded with vertices and edges generated by both the prior random sampling and the local optimization. 

\subsection{Global Exploration using Sampling-based Motion Planning}
\label{sec:sampling}

The global exploration step uses a sampling-based motion planner to expand a graph until a new path---spanning from the start configuration to any goal configuration---is found that is of lower cost than any previously found path. The sampling-based motion planner maintains a graph $G = (V, E)$, where $V$ is a set of vertices which represent collision-free configurations of the robot and $E$ is a set of edges, where an edge represents a collision-free transition between two robot configurations.

To expand the graph $G$, our method is designed to use an asymptotically optimal sampling based motion planner such as $k$-nearest Probabilistic Roadmap (PRM*) \cite{Karaman2011_IJRR}. PRM* samples random configurations in the robot's configuration space, locates their $k$ nearest neighbors (where $k$ changes as a function of the number of vertices in the graph), and attempts to connect the configurations to each of its neighbors (connection step).
In each graph expansion step, PRM* executes until a path better than the current best is found, at which point the optimization step begins.
We require that the algorithm randomly samples and attempts to connect at least one vertex in between consecutive optimization steps, to ensure that the global exploration continues.
We refer to IOS-MP with $k$-nearest PRM* for its sampling-based planning step as IOS-MP (PRM*).

IOS-MP can alternatively be used with other asymptotically optimal motion planners, such as Batch Informed Trees (BIT*) \cite{Gammell2015_ICRA}. 
BIT* operates by processing samples in batches.
From a batch of samples it builds a tree using a graph search heuristic during tree construction until a solution is found or the tree can no longer be expanded.
For the next batch, it limits its sampling to the subspace in which a solution of higher quality could be found. 
For IOS-MP with BIT*, the usual BIT* algorithm executes for a short amount of time without interruption ($\le 0.2$ seconds) to take advantage of its batching properties. After the short execution, the best path found by BIT* is evaluated against the previous best, and if it has improved, it is optimized. 
We refer to IOS-MP with BIT* for its  sampling-based planning step as IOS-MP (BIT*).
In Sec.\ \ref{sec:results}, we show results for both when PRM* and when BIT* are used within IOS-MP.

\subsection{Local Optimization using an Augmented Lagrangian Method}

For the local optimization step, we use a nonlinear constrained optimization method called the augmented Lagrangian (AL) method to locally optimize the path. 
AL is similar to standard Lagrangian methods, in that it utilizes Lagrangian multipliers, but differs in that it adds additional quadratic constraint terms.
AL is also similar to penalty-based optimization methods, but by introducing explicit Lagrange multiplier estimates at each step of the optimization, in practice it is able to reduce ill conditioning \cite{Wright_1999}.

We apply AL to the nonlinear constrained optimization problem in \eqref{eqn:PlanningFormulation}, where the initial value for $\vec{p}$ is the most recent path found by the sampling-based method. 
The AL method iteratively minimizes the augmented Lagrangian function, $\mathcal{L}_A$ defined in \eqref{eq:L} and \eqref{eq:psi}, based on each constraint $g_i$ in the set of constraints $I$, and an iterative approximation of the Lagrangian multipliers $\vec{\lambda}$.

\begin{equation}\label{eq:L}
\mathcal{L}_A(\vec{p}, \bm{ \lambda }^k , \mu_k) = f(\vec{p}) + \sum_{i \in I} \psi(g_i(\vec{p}), \lambda_i^k, \mu_k) \enspace .
\end{equation}
\begin{equation}\label{eq:psi}
\psi(\gamma, \sigma, \mu) = \begin{cases} 
- \sigma \gamma + \frac{1}{2\mu}\gamma^2 &\enspace \enspace  \gamma-\mu \sigma \le 0  \\
-\frac{\mu}{2}\sigma^2 & \enspace \enspace\mbox{otherwise}\\
\end{cases}
\enspace .
\end{equation}

At the $k^{th}$ step of the AL algorithm, outlined in Alg. \ref{alg:AL}, $\mathcal{L}_A$ is minimized with respect to the trajectory $\vec{p}_k$. 
In our implementation we use gradient descent with line search for this minimization.

\begin{algorithm}
\KwIn{$\textbf{p}_0$, $\mu_0 > 0$, $\tau > 0$, $\bm{\lambda}_0$}
\KwOut{optimized path $\textbf{p}$}
$k \leftarrow 0$\\
\While{$\|\bigtriangledown_\textbf{p}\mathcal{L}_A\| \ge \tau$} {
$\vec{p}_{k+1}:= $minimize w.r.t. $\vec{p}_k$, $\mathcal{L}_A(\vec{p}_k, \bm{\lambda}, \mu)$ \;
Update $\bm{\lambda}$\;
$\mu := \mu \cdot \mu_{up} \mbox{, where } \mu_{up} \in (0, 1)$\;
$k \leftarrow k+1$\;
}
Return {$\textbf{p}_{k-1}$}\;
\caption{The Augmented Lagrangian method}
\label{alg:AL}
\end{algorithm}

After the minimization, for each inequality constraint $i \in I$, its associated Lagrangian multiplier $\lambda_i$ is then updated using the following formula.

\begin{equation}\label{eq:lambda}
\lambda^{k+1}_i = \texttt{max}\left(\lambda^k_i - \frac{g_i(\vec{p}_k)}{\mu_k}, 0\right)
\end{equation}

Penalty multiplier $ \mu$ is then updated by multiplying it by a parameter $\mu_{up}: 0 < \mu_{up} < 1$.
The algorithm then updates the augmented Lagrangian function and iteration continues until convergence of $\vec{p}$.

For IOS-MP, $f(\vec{p})$ is path length, and the set of inequality constraints $I$ in \eqref{eqn:PlanningFormulation} is the union of the problem-specific inequality constraints $J$ and the obstacle avoidance constraints. For obstacle avoidance, the set of inequality constraints $K$ consists of a constraint per obstacle primitive per pair of adjacent configurations in $\vec{p}$.
For a path $\vec{p}$ consisting of $k$ configurations, there will be $(k-1) \cdot |O|$ inequality constraints related to obstacle avoidance, as we require the edge between configurations to be collision free.
So for constraint $g_i$, associated with obstacle primitive $o$ and configurations $\vec{q}_k$ and $\vec{q}_{k+1}$, $g_i(\vec{p}) = \texttt{clearance}(\vec{q}_k, \vec{q}_{k+1}, o)$.

In the case of motion planning from one configuration to another configuration, the AL method does not update these two configurations but rather updates only the configurations between the start and goal.
In this way, the equality constraints $\vec{q}_0 = \vec{q}_\textrm{start}$ and $\vec{q}_{|\vec{p}|} = \vec{q}_\textrm{goal}$ are trivially enforced.
In the case of goal regions, an additional inequality constraint is added to guarantee that the goal configuration remains within the goal region. Also in practice, to guarantee timely exit, we enforce maximum iteration counts for both the internal minimization and the outside AL loop.

\subsection{Analysis}
\label{sec:analysis}

We show that IOS-MP (PRM*), a variant of IOS-MP that uses $k$-nearest PRM* for its sampling-based planning step, is asymptotically optimal. 
To ensure that paths computed asymptotically approach the globally optimal path, we must ensure that adding vertices to the graph based on the optimization step does not affect the asymptotic optimality properties of the sampling-based motion planner.
We do this by making a distinction between vertices added by the sampling-based method and vertices added by the optimization.
We only count vertices added by the sampling-based method toward the $k$-nearest neighbor count when adding edges.
Because of this distinction, we can show that IOS-MP (PRM*) produces a valid supergraph of what PRM* would have produced, were the optimization steps of IOS-MP (PRM*) not performed.
The behavior, as time progresses, of IOS-MP (PRM*) can be conceptualized as the behavior of PRM* with \emph{finite} time interruptions occurring periodically during which the optimization steps are performed.
To ensure that the asymptotic optimality of PRM* is retained by IOS-MP (PRM*), the following assumptions must hold:

\begin{assumption} \label{as:two}
The optimization step is guaranteed to exit in finite time.
\end{assumption}

\begin{assumption} \label{as:five}
The order and configurations randomly sampled by the graph construction algorithm are unchanged from what they would be were there no optimization steps occurring.
\end{assumption}

\begin{theorem} \label{th:io}
For any finite running time $\tau$, by which an unmodified $k$-nearest PRM* motion planning algorithm would have produced graph $G_u$, there exists a finite running time $\hat{\tau}$, by which IOS-MP (PRM*) produces a graph $G_s$ which is a valid supergraph of $G_u$.
\end{theorem}

\begin{proof}
Consider the ordered list $V_u = \{v_0, v_1, \dots, v_k\}$ where $V_u$ represents the order in which vertices are sampled during the construction of graph $G_u = (V_u, E_u)$ up to time $\tau$.
Also consider the ordered list $V_u^i$ as the sub list of $V_u$ at iteration $i$.
Additionally, consider the set of edges $E_u^i$ as the subset of $E_u$ at iteration $i$.
As we are only considering static environments, $E_u^i$ is dependent only on the order and presence of each vertex $v \in V_u^i$ at the time of connection.
That is to say that the edges are not affected by the specific \emph{time} at which a vertex was added, but rather only on the existence of the set of vertices in the graph when the connection is attempted.

Now consider IOS-MP (PRM*)'s construction of graph $G_s = (V_s, E_s)$, and specifically the ordered construction of $V_s^i$ and $E_s^i$, as above.
By assumption \ref{as:five}, each vertex $v_s^i$ will be identical in order and value, up to its counterpart vertex $v_u^i$ in $V_u^i$.
Also, as a result of the connection strategy detailed above that treats optimization-added vertices as distinct when computing $k$-nearest neighbors, $E_s^i$ will contain all of the edges present in $E_u^i$.

Next, consider the sequence of steps taken by the sampling-based algorithm alone during construction up to iteration $i$. 
An example of such a sequence could be: sample vertex 0, sample vertex 1, connection step 0, sample vertex 2, sample vertex 3, connection step 1, \dots, sample vertex $i$.
Note that the frequency of a connection step is dependent on the specifics of the implementation but happens no more than once per sample.

The corresponding sequence of steps taken by IOS-MP (PRM*) will look very similar, for example: sample vertex 0, sample vertex 1, connection step 0, optimization step 0, sample vertex 2, sample vertex 3, connection step 1, optimization step 1, \dots, sample vertex {i}.
Note also that we require there to be at most one optimization step present between consecutive connection steps (see Sect.~\ref{sec:sampling}).
Also, by the connection strategies detailed above, the set of vertices and edges after an optimization step is a superset of the set of vertices and edges prior to the optimization step, and because we only add the vertices and edges back if they compose a valid path, it is a valid graph.

Let the time required to sample vertex $k$ be $\tau_s^k$, and the time required for connection step $j$ be $\tau_c^j$.
The time required to construct the graph up to $i$ sampled vertices, $\tau_g^i$ is then:
\begin{equation}
\tau_g^i = \sum_{n=0}^i \tau_s^n + \sum_{m = 0}^j \tau_c^m \mbox{ where } j \le i \enspace .
\end{equation}

Now consider the time required to complete the $k^{th}$ optimization step, $\hat{\tau}_o^k$.
This time will be finite by assumption \ref{as:two}.
The aggregate time then, to perform IOS-MP (PRM*) up to sampled vertex $i$ then becomes:
\begin{equation}
\hat{\tau}_g^i = \sum_{n=0}^i \tau_s^n + \sum_{m = 0}^j \tau_c^m + \sum_{p = 0}^k \hat{\tau}_o^p \mbox{ where } k \le j \le i \enspace .
\end{equation}
As such, the aggregate time required of IOS-MP (PRM*) up to sampling iteration $i$ will be a finite sum of finite times, and subsequently finite itself.

Therefore, there exists a finite time $\hat{\tau}$, where IOS-MP (PRM*)'s underlying sampling of vertices during construction of $V_s$ as a superset of $V_u$ where the vertices which are present in $V_u$ were added to $V_s$ in the same order as in $V_u$, and the connections $E_s$ were constructed as a superset of the edges $E_u$.
This results in a graph $G_s$ which is a valid supergraph of $G_u$ produced in finite time.
\end{proof}

For any finite time, IOS-MP (PRM*) will produce, in a longer finite time, a supergraph of the planner which would have been produced by PRM* alone, by Thm. \ref{th:io}.
As such, any path that would have been present in PRM*'s graph will also eventually be present in IOS-MP (PRM*)'s graph in finite time.IOS-MP (PRM*) will therefore produce, in finite time, a path that is equal or better than any path that would be produced by the PRM* alone in finite time.
It then follows naturally that as PRM* is both probabilistically complete and asymptotically optimal \cite{Karaman2011_IJRR}, IOS-MP (PRM*) will also be probabilistically complete and asymptotically optimal.

It should be noted that the supergraph assumption does not necessarily hold for IOS-MP (BIT*) in our implementation, and as such we cannot claim that IOS-MP (BIT*) is asymptotically optimal as a result of this proof, but we note that in most cases it works comparably or better than IOS-MP (PRM*) as shown in the next section.

%%%%%%%%%%%%%%%%%%%%%%%%%%%%
% RESULTS
%%%%%%%%%%%%%%%%%%%%%%%%%%%%
% !TEX root =  Kuntz_IOSMP.tex

\section{Results}
\label{sec:results}

\begin{figure*}[!t]
	\includegraphics[width=1\linewidth]{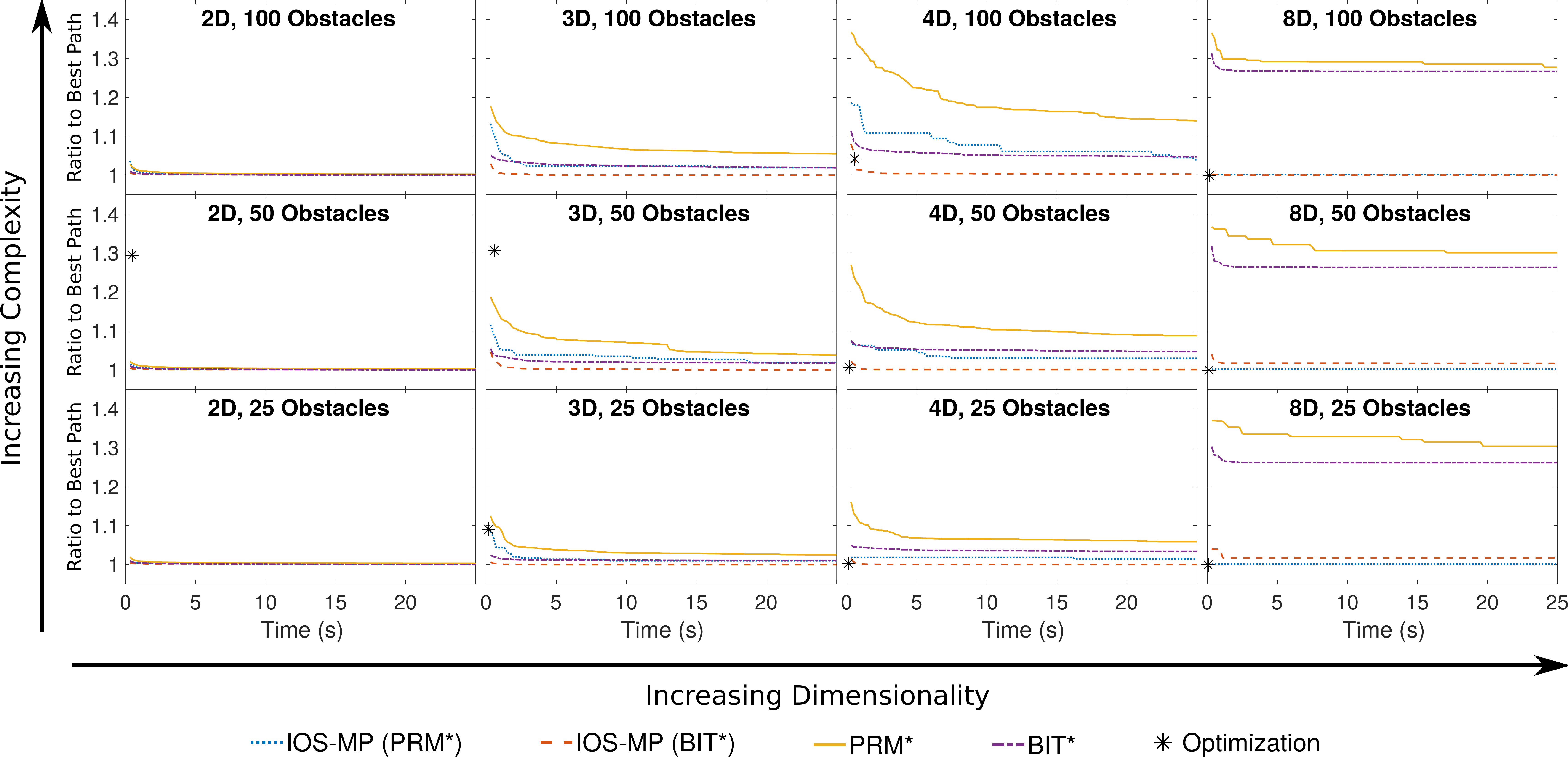}
	\caption{Comparison of methods in randomized environments of varying complexity and dimensionality. In the cases where the asterisk denoting the results from the local optimization is not shown, the value is above the upper end of the plot, between $1.45$ and $1.55$}
    \label{fig:resultsAll}
\end{figure*}

We evaluate the performance of IOS-MP with respect to configuration space dimensionality and varying environmental complexity, as well as demonstrate performance on simulated 7-DOF manipulation tasks based on the Fetch robot \cite{Fetch}. We evaluated two versions of IOS-MP: (1) IOS-MP (PRM*), which uses PRM* as the underlying sampling-based planner, and (2) IOS-MP (BIT*), which uses BIT* as the underlying sampling-based planner. We used implementations of PRM* and BIT* from the Open Motion Planning Library (OMPL) \cite{OMPL} for the sampling-based motion planning step of IOS-MP. We compared our two variants of IOS-MP with methods based solely on sampling-based motion planning (PRM* and BIT* as implemented in OMPL) or based solely on optimization (an augmented Lagrangian constrained optimization).
All experiments were performed on a system with an 8 core 3.2GHz Intel Xeon E5-2667 processor.

\subsection{Dimensionality and Complexity Scaling}
To evaluate how the method scales with respect to both configuration space dimensionality and environmental complexity, we evaluate IOS-MP for a point robot in a unit box of dimension $d$. We varied $d$ from 2 to 8 and varied the number of obstacle primitives from as few as 25 to as many as 100. 
In each environment of dimension $d$ we randomly placed obstacles in the form of hyperspheres of dimension $d$, i.e., circles for 2D, spheres for 3D, and hyperspheres for 4D and 8D. The radii of each obstacle was a random number sampled from a uniform distribution in the range [0.05,0.2]. We defined the start and goal configurations for the point robot as points at the centers of opposite faces of the box. 
In the optimization formulation for \eqref{eqn:PlanningFormulation}, we used inequality constraints for clearance from the hypersphere obstacles.
An example environment for 2D is shown in Fig.\ \ref{fig:fig1}.

We show the relative performance of IOS-MP (PRM*), IOS-MP (BIT*), PRM*, BIT*, and Optimization (based on a local augmented Lagrangian method) for scenarios of different dimensionality and environment complexity in Fig. \ref{fig:resultsAll}. 
In each plot, we ran the IOS-MP variants and the sampling-based methods in an anytime manner for 25 seconds, considering the best solution found over time and averaging over runs in 15 environments. 
For Optimization, we initialized the algorithm with a straight line path in configuration space from start to goal and a fixed number of trajectory points (20) and ran the algorithm until convergence for each of the 15 environments.
Optimization does not always succeed in finding a solution (it was successful in $86\%$ of the point robot scenarios represented here); because of this, we computed averages only across runs for which a collision-free solution was found and plot a $*$ for this average.
To display meaningful averages across the 15 environments for each graph, we defined the ``best solution'' for a particular environment as the shortest path found by any method at any time for that environment, and then averaged over ratios with respect to the best solution in each environment.

In Fig. \ref{fig:resultsAll}, both IOS-MP (BIT*) and IOS-MP (PRM*) outperform their sampling only counterparts in every case, and outperform optimization-based motion planning in most cases, with IOS-MP (BIT*) being the best all around method.
The results show that the sampling-based methods PRM* and BIT* perform relatively well in the lower dimensionality cases, especially where the environment is complex (i.e. many obstacle primitives), while local optimization performs increasingly well as the dimensionality increases.
By interleaving sampling-based methods and local optimization, IOS-MP gains the speed of both of these methods on these different types of problems.
When looking at the trend from left to right of Fig.\ \ref{fig:resultsAll}, dimensionality increases, and the percent of configuration space that is free also rises because the number and average radius of the obstacles is held constant.
Because local optimization is relatively more efficient at reducing the length of a path in higher dimensional configuration spaces that are more sparse, IOS-MP is particularly effective for these higher dimensional problems compared to the motion planners that only use sampling-based methods. 

\subsection{7-DOF Manipulation Tasks}
We demonstrate IOS-MP's efficacy in various 7-DOF manipulation tasks with the arm of a Fetch robot. 
\subsubsection{Motion Planning to a Goal Configuration}
In the first scenario (Fig. \ref{fig:fetch}), the Fetch robot must plan a motion from its start configuration, where its arm is beneath a ladder to grasp a tool, to a goal configuration, where its arm is above the ladder, as if handing off the tool to a person above.
The ladder and robot links were represented by 8 capsule primitives. In the optimization formulation for \eqref{eqn:PlanningFormulation}, we used inequality constraints for obstacles avoidance, to represent joint limits, and to guarantee the robot does not self-collide.

\begin{figure}[t]
  \centering
  \subfloat[PRM* Path]{\label{fig:fetchPRM}\includegraphics[width=0.35\linewidth]{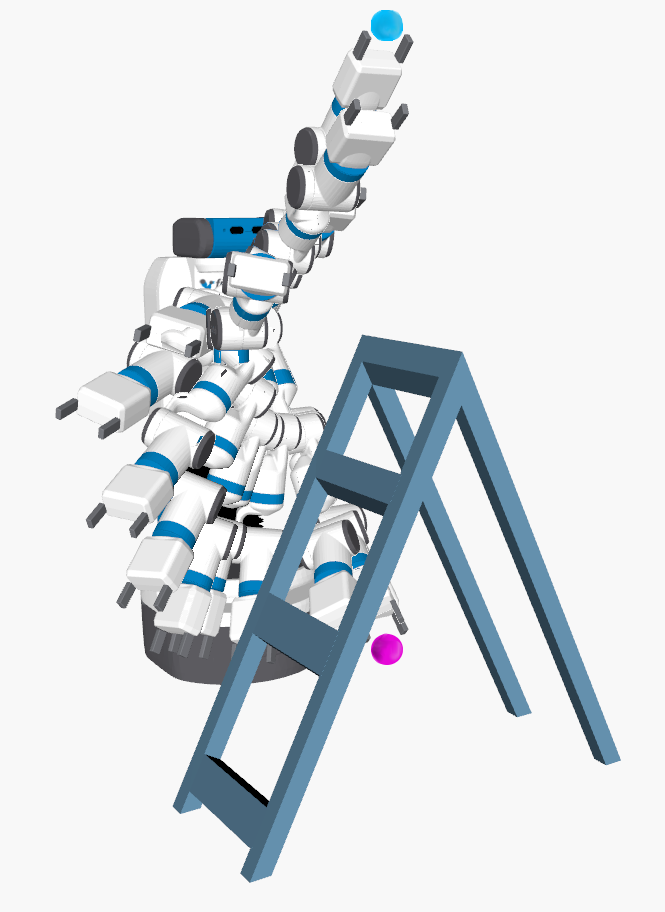}}  \hspace{0.25em}
  \subfloat[Optimized Path]{\label{fig:fetchOPT}\includegraphics[width=0.35\linewidth]{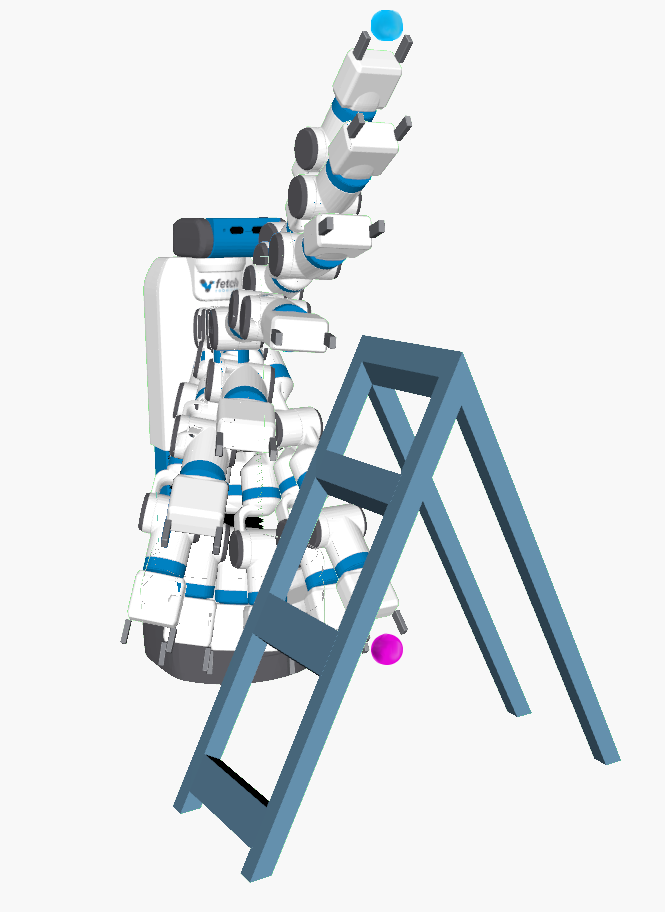}} \\
    \subfloat[Method Comparison]{\label{fig:fetchGraph}\includegraphics[width=0.7\linewidth]{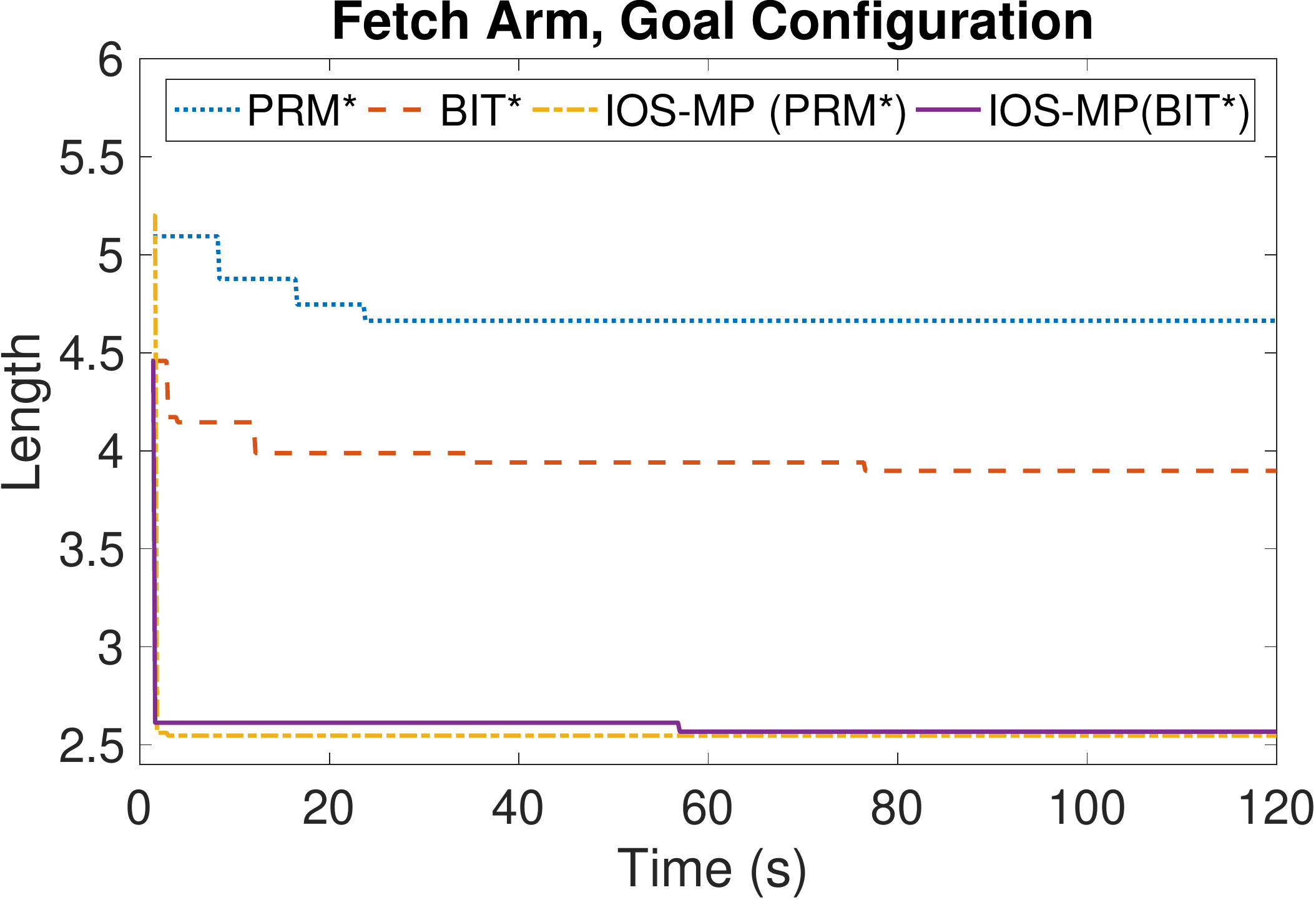}}
    \caption{(a-b) Visualization of the first Fetch robot arm manipulator task, start configuration in pink and goal configuration in blue, (a) with a path found by PRM* (b) and its optimized equivalent. (c) A comparison of the performance of the 4 methods for the Fetch task.}
    \label{fig:fetch}
\end{figure}

Fig. \ref{fig:fetchGraph} shows the results of applying PRM*, BIT*, IOS-MP (PRM*), and IOS-MP (BIT*) to this motion planning problem. The methods were run for 120 seconds, recording the path length based on Euclidean distance in configuration space of the best path found up to that time.
Additionally, we evaluated augmented Lagrangian local optimization, initialized with a 20-point straight line path in configuration space, but this method did not converge to a valid collision-free solution.

In this scenario, both IOS-MP (PRM*) and IOS-MP (BIT*) drastically outperform their sampling-only counterparts.
In just a couple of iterations, the loop of sampling-based methods and local optimization produces better plans than other methods produce in 120 seconds. 
The high dimensionality of the problem means that the optimization step of IOS-MP can provide a large benefit in a short time.
Fig. \ref{fig:fetch} illustrates (a) a path found during the graph expansion step of IOS-MP (PRM*), and (b) the path after the optimization step of the same iteration of IOS-MP. The optimization step of IOS-MP significantly reduces extraneous motion, and iterating with both a sampling-based step and optimization step allows for fast improvement toward a globally optimal solution.

\subsubsection{Motion Planning to a Workspace Goal Region}
In the second scenario, the Fetch robot is placed in a more cluttered scene (see Fig.\ref{fig:fetchRegion}) and must plan a path starting from a start configuration, in which it is reaching up as if to grab a tool from a person on the ladder, and moving such that its end-effector enters a goal \emph{region} in the workspace under the ladder.
We randomly perturb the location of the workspace goal region under the ladder and average the results of 80 runs.
The ladder, robot, and other obstacle geometries are being represented in this scene using 16 capsule obstacle primitives.
In this scenario, the goal is not a specific configuration, but a potentially infinite number of configurations, all configurations in which the end effector is within a workspace region.
In fact, in configuration space the goal region may be both non-convex and disconnected.
To solve this motion planning problem, the robot must consider many possible paths across multiple homotopic classes that may pass between different rungs or around the side of the ladder.

During execution, IOS-MP (PRM*) incrementally discovers and explores the goal region as the PRM* graph is constructed.
As configurations which lay within the goal region are sampled, they are added to the set of goal configurations to which the planner is attempting to connect to from the start configuration.
In the optimization formulation for this scenario, we add an additional inequality constraint, the end effector's proximity to the workspace goal region, which requires the end effector to remain within the workspace goal region during optimization, however the goal configuration is allowed to vary during the optimization as long as the constraint is respected at convergence.
Due to goal region restrictions in the OMPL 1.2.1 BIT* implementation, we only evaluate IOS-MP (PRM*) and PRM* for this scenario.
To evaluate the efficacy of sharing information from the optimizer with the roadmap we show results for which the optimized paths from IOS-MP are added back into the roadmap and results for which they are not.
To average multiple runs we show results beginning at the instant the first path was found.

\begin{figure}[t]
  \centering
  \subfloat[]{\label{fig:fetchR1}\includegraphics[width=0.33\linewidth]{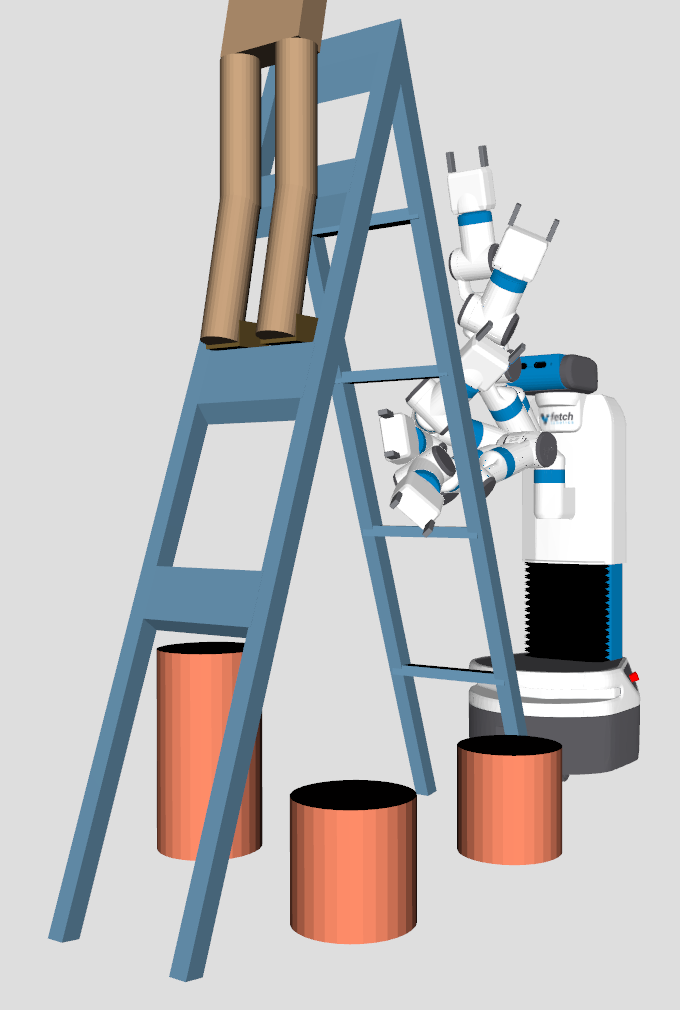}} \hspace{0.25em}
  \subfloat[]{\label{fig:fetchR2}\includegraphics[width=0.33\linewidth]{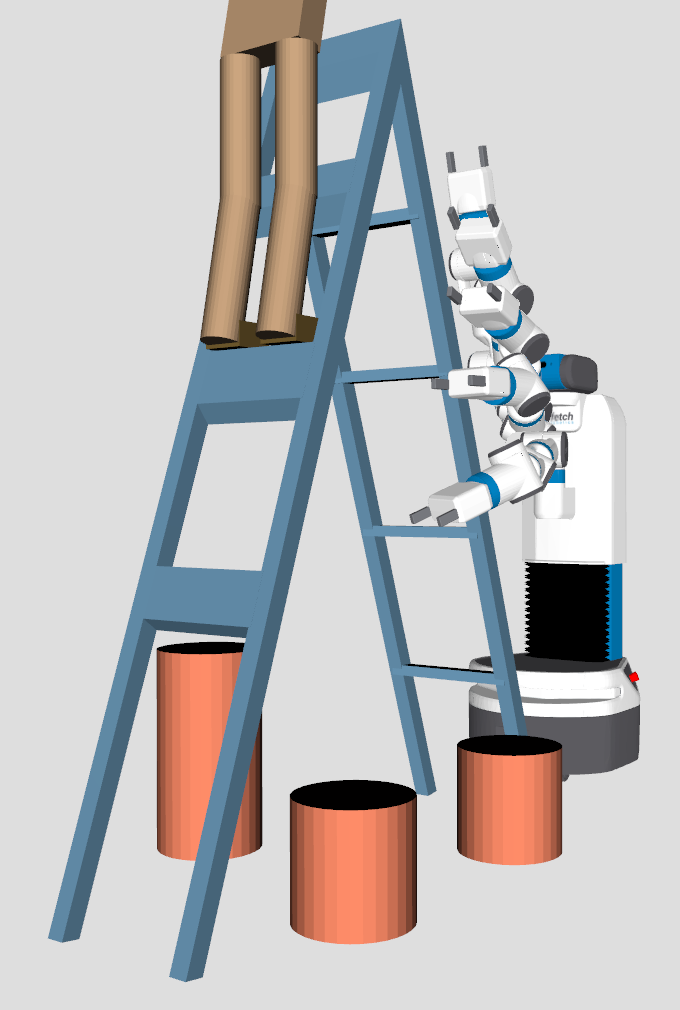}} \\
    \subfloat[]{\label{fig:fetchR3}\includegraphics[width=0.7\linewidth]{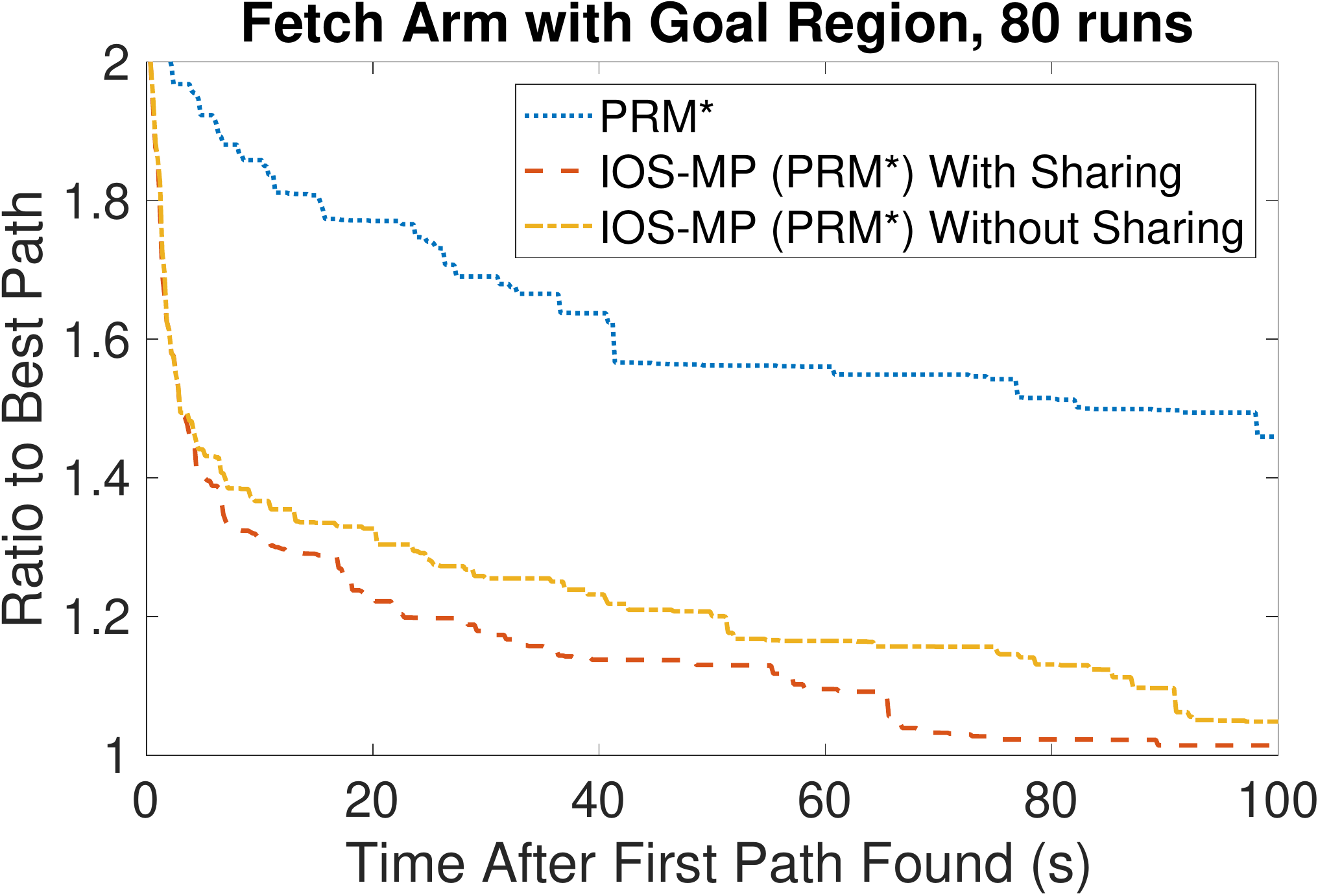}}
    \caption{(a-b) A task in which the Fetch robot must move its end effector from above its head to a region underneath the ladder. (a) A path found during an early iteration of IOS-MP in a non-ideal homotopic class between ladder rungs. (b) A path in a better homotopic class found in a later iteration of IOS-MP. (c) A comparison of the performance of IOS-MP (PRM*), with and without optimized path sharing, and PRM*. We average results over 80 runs in which the workspace goal region was randomly perturbed under the ladder.}
        \label{fig:fetchRegion}
\end{figure}

Fig. \ref{fig:fetchRegion}(a-b) depicts paths which exist in two separate homotopic classes, demonstrating the effectiveness of the global nature of the sampling-based planning while still benefiting from the optimization of the paths found.
Fig. \ref{fig:fetchRegion}(c) also shows that IOS-MP (PRM*) performs very well compared to the sampling-based planner alone, and demonstrates the substantial benefit of adding the optimized path back into the roadmap.

%%%%%%%%%%%%%%%%%%%%%%%%%%%%
% CONCLUSION
%%%%%%%%%%%%%%%%%%%%%%%%%%%%
% !TEX root =  Kuntz_IOSMP.tex

\section{Conclusion}

In this paper, we present a method designed to achieve the best of both local optimization and global sampling-based motion planning.
Through interleaving local optimization with global exploration, and sharing valuable information between the two, our method works in an anytime fashion, providing locally optimized solutions as of the most recent optimization step, while still providing global asymptotic optimality.
We demonstrate that our method, integrated with both PRM* and BIT*, performs well compared to local optimization or sampling-based methods alone in experiments of varying environmental complexity and dimensionality.
In the future, we plan to investigate integrating IOS-MP with other optimization methods, such as sequential quadratic programming, using cost functions other than path length, and using different obstacle primitives.

\bibliographystyle{plain}
\bibliography{references-fixed,references}

\end{document}